\relax
\documentclass[letterpaper]{article} 
\usepackage{times}  
\usepackage{helvet} 
\usepackage{courier}  
\usepackage[hyphens]{url}  
\usepackage{graphicx} 
\usepackage{graphicx}  
\usepackage{url}            
\usepackage{booktabs}       
\usepackage{amsfonts}       
\usepackage{nicefrac}       
\usepackage{microtype}      
\usepackage{graphicx}
\usepackage{booktabs} 
\usepackage[american]{babel}
\usepackage{times}
\usepackage{helvet}
\usepackage{courier}
\usepackage{float}
\usepackage{subfigure}
\usepackage{ifthen}
\usepackage{amssymb}
\usepackage{amsthm}
\usepackage{supertabular}
\usepackage{array}
\usepackage{multirow}
\usepackage{psfrag}
\usepackage{amsmath}
\usepackage{hhline}
\usepackage{fancybox}
\usepackage{ctable}
\usepackage{bm}
\usepackage{algorithm}
\usepackage{algorithmic}
\usepackage{placeins}
\usepackage{setspace}
\usepackage{enumerate}
\usepackage{multirow}
\usepackage{enumitem}
\usepackage{url}
\usepackage{epstopdf}
\usepackage{epsfig}
\usepackage{thmtools}
\usepackage{thm-restate}
\usepackage{mathtools}
\usepackage{soul}
\usepackage{enumitem}

\DeclareMathOperator*{\argmin}{arg\,min}

\newcommand{\R}{\mathbb{R}}

\newtheorem{assumption}{Assumption}

\title{Collaborative Inference for Efficient Remote Monitoring}
\author{Chi Zhang$^1$, Yong Sheng Soh$^1$, Ling Feng$^1$, Tianyi Zhou$^1$, Qianxiao Li$^{1,2}$
\\}
\date{%
	$^1$ Institute of High Performance Computing, Singapore\\%
	$^2$National University of Singapore, Singapore\\[2ex]%
}
 
\begin{document}

\maketitle

\begin{abstract}
While current machine learning models have impressive performance over a wide range of applications, their large size and complexity render them unsuitable for tasks such as remote monitoring on edge devices with limited storage and computational power. A naive approach to resolve this on the model level is to use simpler architectures, but this sacrifices prediction accuracy and is unsuitable for monitoring applications requiring accurate detection of the onset of adverse events.
In this paper, we propose an alternative solution to this problem by decomposing the predictive model as the sum of a simple function which serves as a local monitoring tool, and a complex correction term to be evaluated on the server. A sign requirement is imposed on the latter to ensure that the local monitoring function is safe, in the sense that it can effectively serve as an early warning system.
Our analysis quantifies the trade-offs between model complexity and performance, and serves as a guidance for architecture design. We validate our proposed framework on a series of monitoring experiments, where we succeed at learning monitoring models with significantly reduced complexity that minimally violate the safety requirement.  More broadly, our framework is useful for learning classifiers in applications where false negatives are significantly more costly compared to false positives.
\end{abstract}

\section{Introduction}
Deep learning has been shown to be an effective solution in analyzing time-series data and making continuous predictions~\cite{fawaz2018deep}. It redefines state-of-the-art performance in a wide range of areas, including illness prediction, surveillance monitoring and anomaly detection~\cite{zhao2016deep,acharya2018deep,schirrmeister2017deep,yildirim2018arrhythmia,javaid2016deep,xu2015learning}. As a canonical example, results in~\cite{rajpurkar2017cardiologist} show that deep learning based algorithm already exceeds board certified cardiologists in detecting heart arrhythmias from electrocardiograms when given a large annotated dataset.


The success of deep learning often relies on the large model size and therefore heavy computation complexities. For example, the aforementioned arrhythmias detection algorithm~\cite{rajpurkar2017cardiologist} relies on a 34-layer convolutional neural network to map the sequences of ECG samples to rhythm classes on a fixed dataset. However, for most real scenarios, monitoring is performed on edge devices that are often designed with limited storage and computation ability, such as wearable devices for health monitoring and cameras for anomaly detection. Besides, the input signals for real-time monitoring arrive in a sequential order rather than lie in a fixed dataset, making it more demanding or even impossible for edge devices to provide real-time predictions if a huge network is deployed.


One possible solution is collecting data on edge devices and providing real-time response with a truncated or compressed model. But this would often lead to an undesirable sacrifices of prediction accuracies as most simple neural networks cannot fully reveal the underlying mechanisms and fail to provide predictions as accurate as those complex models.
Another possible approach is to consider the local monitoring devices only as data collectors, while model training and data analysis are performed on remote servers. But this may potentially incur large overheads if data dimension and/or the monitoring time is large. Moreover, from the users' perspective, constantly sending local data like their bio-medical information to remote server may raise their concerns on privacy and security. In the meantime, servers may need to  make continuous predictions and provide timely response for each user, which can often be impractical since the number of users can be huge.

Therefore, it is necessary to design a new learning paradigm with \emph{collaborative inference} so that most of the computations and predictions are performed through the local device,  without the necessities of sending sensitive data to servers. In case of emergency or situations that are beyond the capability of local simple models, data are transported to servers where further analysis is performed. Moreover, for scenarios like medical monitoring, one needs more than just finding an adequate approximation to a target function.  A \emph{safety requirement} should be guaranteed so that local device can detect unusual situations in advance, and false negative cases should always be eliminated. 


\begin{figure*}[h]
	\centering
	\includegraphics[width=0.9\linewidth]{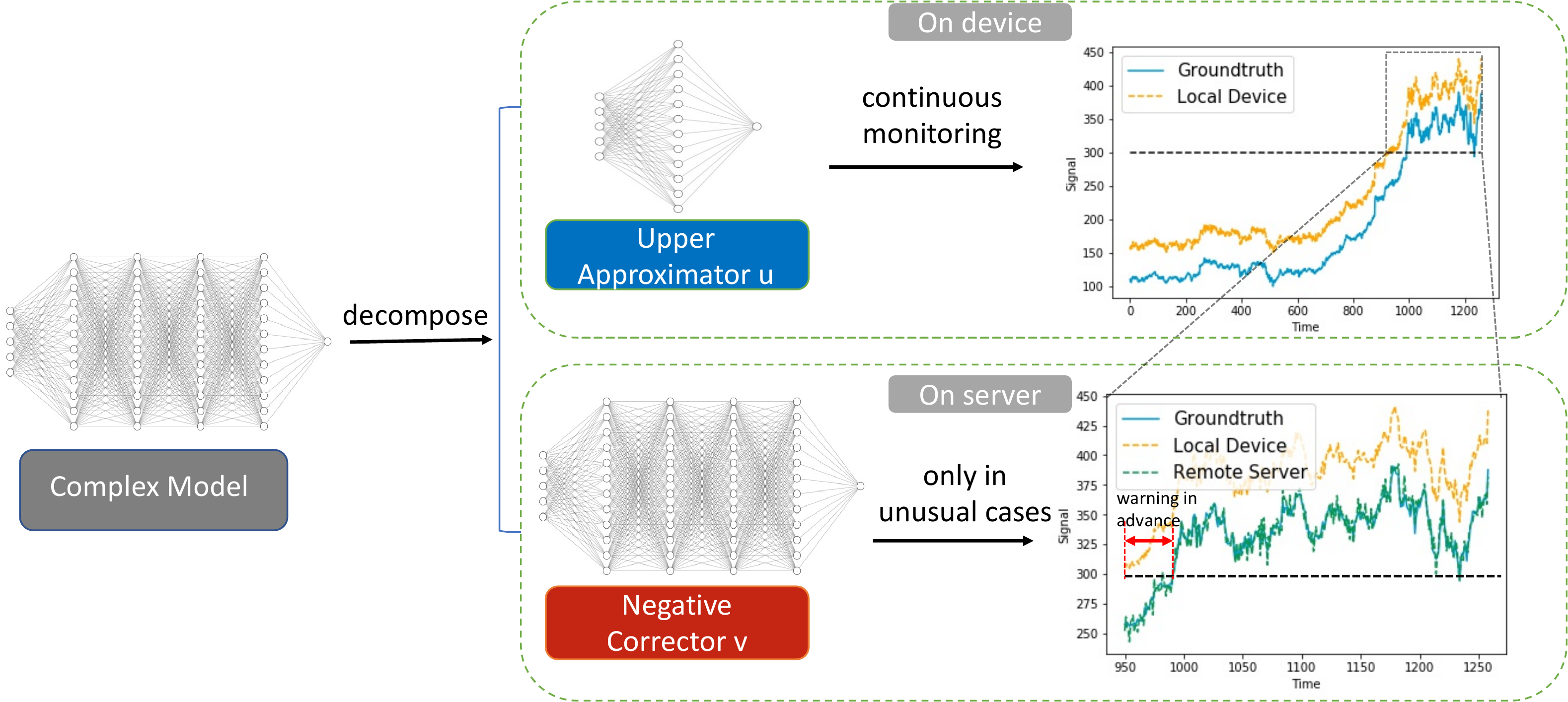}
	\caption{The strong but complex model is decomposed into one on-device upper approximator and one on-server negative corrector. The on-device low-dimensional classifier $u$ provides continuously approximation that is always above the groundtruth to guarantee safety. In case the on-device predictions exceed the threshold, an 
		on-server high-dimensional corrector $v$ is activated to bridge the gap between device prediction $u$  and true signal $f$ and provide more accurate analysis.}
	\label{fig:scheme}
\end{figure*}

To address these two issues, we consider the following question: \textit{Given a strong but complex classifier $v^*$, how to make it suitable to deploy on edge devices and at the same time making sure the safety regime is guaranteed?}

The answer provided in this paper is designing a new learning architecture by decomposing the monitoring system into two separate parts: a simple on-device model acting as the monitoring tool to provide timely response locally and a complex remote model acting as the corrector to provide accurate modifications for unusual cases~(Fig~\ref{fig:scheme}). With this design, the whole system is able to achieve the following advantages. 
\begin{enumerate}
	\item Approximation accuracy: with assistance of the on-server corrector, we show that the approximation ability of this collaborative system is no worse than the original accurate but complex model~(See Prop~\ref{thm:approxpower}). This is in contrast to the classical model compression methods where the system usually involves trade-off between the model size and the accuracy drops. 
	\item Communication reduction: most of the time the computations and predictions are performed through the local device without the necessity of sending sensitive data to servers. The communication will be triggered 
	only when the local predictions exceed the predefined threshold. 
	\item Safety requirement: the correction term on-server is designed to be everywhere non-positive so local model provides upper-bounded predictions and guarantees the safety requirement. 
\end{enumerate}

\paragraph{Related Work} Our work follows the studies of model compression so that complex deep learning algorithms can be deployed on edge devices. Typical model compaction approaches include  distillation~\cite{Bucilua:2006:MC:1150402.1150464,ba2014deep,romero2014fitnets}, pruning~\cite{hanson1989comparing,gong2014compressing,han2015learning}, compact convolution filters~ \cite{zhai2016doubly,cohen2016group} and dynamic execution~ \cite{DBLP:journals/corr/abs-1806-07568,DBLP:journals/corr/abs-1811-01476}. But in contrast to these studies that only focus on model compression and may sacrifice prediction accuracies, we reuse the original model architecture on server and show that  the approximation ability of our scheme is no worse than the original model through collaboration.

Collaborative inference between the server and edge devices  is also studied in~\cite{kang2017neurosurgeon,li2018edge,choi2018deep} by using various computation partitioning strategies. The main difference between our work and these studies is the safety regime proposed in our work, as we strictly require the local predictions to lie above the groundtruth~(See Fig~\ref{fig:scheme}), making our approach more suitable for scenarios like illness prediction and anomaly detection where safety is more vital than good approximation. 


Our work is also partially related to the studies of 
boosting algorithms~\cite{friedman2001greedy,chen2016xgboost,schwenk2000boosting} and cascading methods~\cite{gama2000cascade,zhao2004constrained,marquez2018deep}, as we both use multiple classifiers to form a strong classifier. The difference between our work and these works lies in the fact that boosting and cascading algorithms usually assume we have only weak classifiers and focus on forming a strong classifier by cascading the decisions of these weak classifiers. But our work assumes that the strong classifier already exists, but unsuitable for edge devices due to its large model size and heavy computation requirement. Therefore, we decompose the strong classifier into one weak monitoring classifier on device and one strong correcting classifier on server.

%

\section{Efficient Monitoring by Model Decomposition}
The previous conceptual scheme shall be formally formulated in this part, followed by three evaluation metrics: approximation error, false positive rate and false negative rate.

\subsection{Problem Formulation}
The introduced model decomposition scheme in the previous section can be summarized as the task of computing a function approximation.  Formally, we let $\Omega \subset \R^d$ be the bounded set of all possible states (e.g. physiological data collected by a wearable device), and let the function $f : \Omega \rightarrow \R$ denote the ground truth target, which maps the input to a scalar output (e.g. some metric that measures the health index of an individual, based on the physiological data). Note that the output of $f$ can also be binary (-1 and 1), in which case we have a classification problem. We assume that an adverse event (e.g. onset of a heat stroke) happens when $f > \gamma$ for some threshold $\gamma$, which for simplicity of presentation we can set to $0$. In the binary classification case, an adverse event is simply $f=1$.

The goal of efficient remote monitoring is to learn an approximation $\hat{f}$ that is accurate, deployable on edge devices and most importantly, safe -- the approximation should signal an adverse event before or when it happens. Moreover, we want the approach to be easily implementable. To this end, we assume that we already have a good hypothesis space $\mathcal{V} \subset C(\Omega)$ ($C(\Omega)$ denotes the space of continuous functions on $\Omega$) that approximates $f$ well, i.e. $\min_{v\in\mathcal{V}} || f - v ||_{\infty} \ll 1$. However, $v^*=\argmin_v || f - v ||_{\infty}$ can be very complex (e.g. deep convolution neural networks), so that $v^*$, despite being able to approximate $f$ well, cannot be directly deployed for monitoring. The question we would like to answer is, given $\mathcal{V}$, can we construct a new hypothesis space so that we can realize efficient monitoring?

\subsection{Model Decomposition}
The approach we propose in this paper is as follows. Consider in addition to $\mathcal{V}$ another hypothesis space $\mathcal{U} \subset C(\Omega)$ that is very simple, so that any function $u$ from it is deployable on an edge device. For example, $\mathcal{U}$ may consist of much smaller neural networks than those in $\mathcal{V}$. Now, we form as an approximation to $f$ in the following form
\begin{align}\label{eq:fhat_construction}
\hat{f} := u - s \sigma(v),
\qquad
u \in \mathcal{U},
\quad
v \in \mathcal{V},
\quad
s >0.
\end{align}
Here, $\sigma : \R \rightarrow (0,1)$ is chosen to be some fixed continuous invertible function whose purpose is to modulate the output from the function $v$ so that it is contained within a bounded interval, whose scale is set by $s$; for instance, a convenient choice for $\sigma$ is the sigmoid function $ 1/(1+\exp(-x))$. For the purpose of theoretical analysis, we hereafter assume that $\sigma$ is twice differentiable with bounded second derivative. Note that these conditions are satisfied by the sigmoid function.

Now, suppose that $u,v,s$ are chosen such that $\hat{f} \approx f$. Then, by construction~\eqref{eq:fhat_construction} we must have $u \geq \hat{f} \approx f$, since both $s$ and $\sigma(v)$ are necessarily positive. Hence, $u$ satisfies both the \emph{efficiency condition} ($\mathcal{U}$ is simple) and the \emph{safety condition} ($u \geq f$), and is a promising candidate for remote monitoring on edge devices. On the other hand, the second term in~\eqref{eq:fhat_construction} serves as a correction on the server side and need only be evaluated when accurate predictions are required (e.g. when $u$ is above the threshold, signaling the possibility of unusual cases).

\subsection{Performance metrics}
The following performance metrics are of interest both in theory and in practice:
\begin{enumerate}[itemsep=0cm]
	\item \textbf{Approximation Error.} Take $p \in [1,\infty]$
	\begin{align}\label{eq:metric_approx}
	\begin{split}
	&\| f - \hat{f} \|_p
	\equiv \| f - (u - s\sigma(v)) \|_p
	\end{split}
	\end{align}
	\item \textbf{False Positive Rate.} Take $\epsilon > 0$
	\begin{align}\label{eq:metric_fp}
	\mu_{\mathrm{FP},\epsilon}
	:=
	{
		\int_{\Omega} 1_{\{f(x) < -\epsilon,u(x)>\epsilon\}} dx
	}
	\end{align}
	\item \textbf{False Negative Rate.} 
	\begin{align}\label{eq:metric_fn}
	\mu_{\mathrm{FN,\epsilon}} :=
	{
		\int_{\Omega} 1_{\{f(x) > \epsilon,u(x)< -\epsilon\}}  dx
	}
	\end{align}
\end{enumerate}
Some remarks on the metrics are in order. For approximation we typically consider the $L^1$ ($p=1$) and $L^\infty$ ($p=\infty$) norms. In the definition of false positive and false negative rates, we introduced a regularization parameter $\epsilon$ representing a margin between the decision boundaries. For understanding sake one may take $\epsilon=0$ but a positive value is useful to rule out pathological cases involving the decision boundary. Also, if $\int_\Omega dx = 1$ then these are true ``rates''. Otherwise they are unnormalized, but this minor point does not affect subsequent results.

\section{Theoretical Analysis}
\label{sec:analysis}
We have established our model decomposition scheme in the previous part, but a few questions still remain unknown: 1) whether this decomposition scheme works well; 2) how to choose the simplified model $u$ and the parameter $s$. The answers shall be provided based on a complete theoretical analysis in the following parts.

\subsection{Approximation Error}
\label{sec:analysis_ae}
The validity of our model decomposition approach rests upon the premise: given that $\mathcal{V}$ approximates $f$ well, does that mean that there must exist $u,v,s$ so that $\hat{f}$ in the form of~\eqref{eq:fhat_construction} also approximates $f$ well? It turns out that the answer is positive under mild conditions on $\mathcal{U}$ and $\mathcal{V}$. 

\begin{restatable}{proposition}{propone}
	\label{thm:approxpower}
	Consider the approximation scheme in~\eqref{eq:fhat_construction}. Suppose that $\mathcal{U},\mathcal{V}$ are linearly closed, meaning that for every $u$ in $\mathcal{U}$ (resp. $\mathcal{V}$), $a u + b \in \mathcal{U}$ (resp. $\mathcal{V}$) for any $a,b\in \R$. Then
	\begin{align}
	\inf_{u\in \mathcal{U}, v\in \mathcal{V}, s>0} \| f - (u - s \sigma(v)) \|_{\infty} \leq \inf_{v\in \mathcal{V}} \| f - v \|_{\infty}.
	\end{align}
\end{restatable}
\begin{proof}
	Let $\epsilon := \inf_{v\in\mathcal{V}}\| f - v \|_\infty$ and fix $\delta > 0$. Then, there exists a $v\in
	\mathcal{V}$ such that $\| f - v \|_\infty \leq \epsilon + \delta$ and $\| v \|_\infty \leq \| f \|_\infty + \epsilon + \delta \leq \| f \|_\infty + 2\epsilon $. By assumption, there exists $z \in \R$ such that $\sigma'(z) \neq 0$. By Taylor's theorem, for each $q\in \R$ and $x\in\R^d$,
	\begin{align*}
	\sigma(z + q v(x)) &= \sigma(z) + \sigma'(z) q v(x) \\
	&+ \frac{1}{2} \sigma''(z + \alpha(x) q v(x)) q^2 {v(x)}^2,
	\end{align*}
	for some $\alpha(x) \in [0,1]$. Let $k := \sup_{x\in\R} |\sigma''(x)|$ and take in the above $q\equiv - 1 / (s\sigma'(z)) $, then we have for any $x\in \Omega$,
	\begin{align*}
	|f(x) - \hat{f}(x)| &\leq | u(x) - s \sigma(z) | + | f(x) - v(x) | \\
	&+ \frac{k}{2 \sigma'(z)^2 s} (\|f\|_\infty + 2 \epsilon)^2.
	\end{align*}
	Now, simply take $s$ large enough so that $\frac{k}{2 \sigma'(z)^2 s} (\|f\|_\infty + 2 \epsilon)^2 \leq \delta$, and then
	take $u(x) \equiv s \sigma(z)$ (which is always possible by linear closure). We then have
	\begin{align}
	|f(x) - \hat{f}(x)| \leq \epsilon + 2\delta.
	\end{align}
	Moreover, $z+qv \in \mathcal{V}$ by the linear closure assumption. Since $\delta$ is arbitrary, this proves the claim.
\end{proof}

\noindent
\textbf{Remark}: The above result shows by collaboration of the on-device model $u$ and the on-server negative corrector $- s\cdot \sigma(v)$, the approximation power is no worse than the original complex model $v^* := \inf_{v\in \mathcal{V}} \| f - v \|_{\infty}$. This is in contrast to previous algorithms where the original model $v$ is discarded and the model's approximation ability usually drops after compression. 


\vspace{0.1in}

Prop~\ref{thm:approxpower} ensures that the combined model $\hat{f}$ to be a good approximator, but it does not imply how to choose the model $u$ and the parameter $s$. To proceed, we need the following assumption.

\begin{assumption}
	\label{assum}
	The underlying groundtruth $f$ admits a decomposition as the linear combination of simpler functions drawn from a collection $\Phi$, namely
	\begin{equation} \label{eq:f_decompseries}
	f = \sum_{i=1}^{\infty} a_i \phi_i ,
	\quad \phi_i : \Omega \rightarrow \R,
	\quad \phi_i \in \Phi
	\end{equation}
\end{assumption}





\noindent
\textbf{Remark}: Any well-behaved function has the above decomposition scheme, including  many practical machine learning models. In neural networks, $\phi_i$ represents the high-level features at the second-last layer and the final classification relies on a linear combination of these feature functions~(e.g., in LeNet~\cite{lecun1998gradient}, the classification on the last layer relies on the 84 high-level feature functions $\phi_i$ in the second-last layer). For explicit kernel methods, $\phi_i$ may represent feature maps, while for Fourier expansions these $\phi_i$'s represent the basis functions.

\vspace{0.1in}

With this assumption, we have a direct way of constructing simple approximations $u$ to $f$ by truncating the series.  To promote the safety condition $u \geq f$, it is necessary to augment the truncated series with an appropriate positive function.  For simplicity, we analyze the setting where we augment the truncated series with a constant term:
\begin{equation} \label{eq:un_truncateplusconst}
u_{n,t} := \sum_{i=1}^{n} a_i \phi_i + t,
\qquad
t > 0.
\end{equation}
Here, $n$ controls the complexity of $u_{n,t}$, which for $t$ large enough will become a safe monitoring function. The following result gives some quantitative estimates in this direction.

\begin{restatable}{proposition}{proptwo}
	\label{thm:residualbound_inf}
	Suppose that $f$ admits the decomposition~\eqref{eq:f_decompseries} and $v^*$ represents a universal approximator so that $\inf_{v \in \mathcal{V}} \| \sigma(v) \|_\infty = 0$. By setting $t=t(n):=\| \sum_{i=n+1}^{\infty} a_i \phi_i \|_{\infty}$ and $s \geq 2t$, we have 
	\begin{align} 
	\label{eq:uniformbound_allcases}
	& \ u_{n,t(n)}  \geq f. \\
	& \inf_{v \in \mathcal{V}} \| f - (u_{n,t(n)} - s \sigma(v)) \|_{\infty}=0. 
	\end{align}
\end{restatable}
\begin{proof}[Proof of Proposition \ref{thm:residualbound_inf}]
	The first bound $u_{n,t(n)} \geq f$ comes directly from the definition $t:= \| \sum_{i=n+1}^{\infty} a_i \phi_i \|_{\infty}$. 
	
	\vspace{0.1in}
	Next, by noting that
	\begin{align*}
	&\quad \inf_{v \in \mathcal{V}}  \| f - (u_n - s \sigma(v)) \|_\infty \\ 
	&= \inf_{v \in \mathcal{V}}  \| s \sigma(v) - (t(n) - \sum_{i=n+1}^{\infty} a_i \phi ) \|_\infty \\
	&\leq Ls  \inf_{v \in \mathcal{V}}
	\left\| 
	v - \sigma^{-1}
	\left(
	\frac{1}{{s}}
	\left(
	t(n) - \sum_{i=n+1}^{\infty} a_i \phi_i
	\right)
	\right)
	\right\|_{\infty},
	\end{align*}
	where the inequality comes from $$t(n) - \sum_{i=n+1}^{\infty} a_i \phi_i  = s \sigma \circ \sigma^{-1} [ (1/s) (t(n) - \sum_{i=n+1}^{\infty} a_i \phi ) ]$$ and $L := \sup_{y} | \sigma'(y) |$.
	
	Noticing we have $s \geq 2t(n)$ and therefore $
	\frac{1}{{s}}
	\left(
	t(n) - \sum_{i=n+1}^{\infty} a_i \phi_i
	\right)
	$ takes value in $(0,1)$. Since $v$ is a universal approximator to match any value in $(0,1)$, the r.h.s equals 0.
\end{proof}

\noindent
\textbf{Remark}: (1) By construction of $t(n)$, $u_{n,t(n)} \geq f$ guarantees the safety requirement as the false negative rate is always 0. 

(2) By assuming for simplicity that $\mathcal{V}$ is an universal approximating class for continuous functions (e.g. very large neural networks), the estimate says that as long as the scale $s$ exceeds $2t(n)$, the infimum on the right hand side is 0, i.e. the combined prediction $\hat{f}$ approximates $f$ well. Compared with Proposition \ref{thm:approxpower} which may require very large $s$, here $s$ only needs to exceed $2t(n)$, which converges to 0 as $n\rightarrow\infty$.

(3) The choice of $t(n) = \| \sum_{i=n+1}^{\infty} a_i \phi_i\|_{\infty}$ captures the \emph{dynamic range} of the residual term.  Subsequently, the choice of $u_{n,t(n)}$ is most useful if the series \eqref{eq:un_truncateplusconst} converges rapidly.  More generally, Proposition \ref{thm:residualbound_inf} is most useful when applied in conjunction with function classes $\Phi$ that allow us to expand the target function $f$ with rapidly decaying coefficients. For example, if $\Phi$ is a class of feature maps, then we want to choose this class so that the ground truth is well-approximated by only a few of the feature maps. This can be done for example by training with regularizers that promote sparsity in the combination coefficients.


\subsection{False Positive Rate}
\label{sec:analysis_fp}
Furthermore, observe that $s$ is intimately connected with the false positive rate and they increase together. The following result makes this precise.
\begin{restatable}{proposition}{propthree}
	\label{thm:fp_l1bound} 
	Let $\hat{f} = u - s\sigma(v)$ and $\|f - \hat{f}\|_{\infty} \leq \delta$, then
	\begin{equation*}
	\mu_{\mathrm{FP},\epsilon}
	\leq
	\frac{1}{2\epsilon} (\delta + s) \mathrm{vol}(\Omega),
	\end{equation*}
	where $\mathrm{vol}(\Omega) := \int_{\Omega} dx$.
\end{restatable}
\begin{proof}
	We have
	\begin{align*}
	2\epsilon \mu_{\mathrm{FP},\epsilon}
	&\leq \int_{\Omega} | f(x) - u(x) |
	1_{\{ u>\epsilon, f<-\epsilon \}}  dx \\
	&\leq \int_{\Omega} | f(x) - u(x) | dx\\
	&\leq
	\int_{\Omega} | f(x) - \hat{f}(x) | + s| \sigma(v(x)) |dx
	\leq
	(\delta + s) \mathrm{vol}(\Omega).
	\end{align*}
\end{proof}
\noindent
\textbf{Remark}: This shows that the upper bound for false positive rate increases with $s$, and smaller $s$ is preferred to minimally incur false positive. To combine the result with Prop~\ref{thm:residualbound_inf}, setting $s=2t$ is sufficient to obtain the minimal false positive cases.

Moreover, if we put into $\delta$ the approximation results in Proposition \ref{thm:residualbound_inf}, then the false positive rate again involves a trade-off between $n$ (which decreases $\delta$) and $s$. In summary, estimates~\ref{thm:residualbound_inf} and~\ref{thm:fp_l1bound} are quantitative estimates of the performance trade-off: for fixed model complexity for the on-device monitoring model (fixed $n$), increasing the scale of the server-side correction ($s$) improves overall model accuracy but also increases false positive rates. On the other hand, for a fixed scale, increasing monitoring model complexity improves both the approximation quality and the false positive rate, at the cost of heavier computation.

\subsection{False Negatives Rates} 
\label{sec:analysis_fn}
Prop~\ref{thm:residualbound_inf} already guarantees the false negative to be always 0 by augmenting a positive constant $t$ based on the size of the residual term $\sum_{i=n+1}^{\infty} a_i \phi_i$ in the infinity-norm. This of course relies on the assumption that this residual term converges uniformly so it is finite, which holds in most practical scenarios.

For completeness of the analysis,  it is of interest to understand the performance of our framework in case this uniform convergence is violated or an offset smaller than $\| \sum_{i=n+1}^{\infty} a_i \phi \|_{\infty}$ is chosen. An immediate consequence  is that we violate the safety requirement $u \geq f$; i.e., we incur a False Negative instance.  The analyses developed in Propositions \ref{thm:residualbound_inf} and \ref{thm:fp_l1bound} are applicable for quantifying the extent by which our safety requirement is violated, as well as the size of the residual error $\| f - \hat{f} \|$, and we describe these result formally in Proposition \ref{thm:l1_violatesatefy}.


\begin{restatable}{proposition}{propfour}
	\label{thm:l1_violatesatefy}
	Suppose that $f$ admits the decomposition \eqref{eq:f_decompseries} and that we choose $u_{n,t}$ as \eqref{eq:un_truncateplusconst}.  We have the following bound:
	\begin{equation*}
	\mu(\Omega_{\mathrm{FN,\epsilon}}) \leq \frac{1}{1/(2\epsilon+t)^2}\| \sum_{i=n+1}^{\infty} a_i \phi_i \|_{2}^{2}.
	\end{equation*}
	Define the region $\Omega_{-t,s-t} := \{ x : -t < \sum_{i=n+1}^{\infty} a_i \phi_i (x) < s-t \}$.  We then have the following bound on the residual error $\min_{v \in \mathcal{V}} \| f - (u_{n,t} - s \sigma(v)) \|_{1}$ is bounded above by the following:
	\begin{eqnarray*}
		& & s L \inf_{v \in \mathcal{V}} \left\| 1_{\Omega_{-t,s-t}} \times \left( v - \sigma^{-1} \left( \frac{t - \sum_{i=n+1}^{\infty} a_i \phi_i }{s} \right) \right) \right\|_{1} \\
		& + & \left( \sqrt{ \frac{1}{t^2} + \frac{1}{(s-t)^2} } + \max(t,|s-t|) \left(\frac{1}{t^2} + \frac{1}{(s-t)^2} \right) \right) \| \sum_{i=n+1}^{\infty} a_i \phi_i \|_{2}^{2}.
	\end{eqnarray*}
	In particular, if we pick $t = s/2$, then the latter term is $\lesssim (1/s) \| \sum_{i=n+1}^{\infty} a_i \phi_i \|_{2}^{2}$.
\end{restatable}

\begin{proof}[Proof of Proposition \ref{thm:l1_violatesatefy}]
	The first bound follows from an application of the Chebyshev's inequality.  Next we evaluate the integral of the residual error over the regions $\Omega_{-t,s-t}$ and $\Omega_{-t,s-t}^{c}$, and apply the triangle inequality to arrive at the bound
	\begin{eqnarray*}
		& & \| f - (u_{n,t} - s \sigma(v)) \| \\
		& \leq & \| 1_{\Omega_{-t,s-t}} \times (f - (u_{n,t} - s \sigma(v))) \| + \| 1_{\Omega_{-t,s-t}}^{c} \times ( f - (u_{n,t} - s \sigma(v)) ) \|.
	\end{eqnarray*}
	For the first term we have
	\begin{eqnarray*}
		& & \left\|1_{\Omega_{-t,s-t}} \times \left( s \sigma(v) - ( t - \sum_{i=n+1}^{\infty} a_i \phi_i ) \right) \right\| \\ & \leq & s L \left \| 1_{\Omega_{-t,s-t}} \times \left( v - \sigma^{-1} \left( \frac{t - \sum_{i=n+1}^{\infty} a_i \phi_i }{s} \right) \right) \right\|.
	\end{eqnarray*}
	Note that the latter term is valid because the image of the function $(1/s)(t - \sum_{i=n+1}^{\infty} a_i \phi_i )$ lies within the domain of $\sigma^{-1}$ whenever $x \in {\Omega_{-t,s-t}}$.  For the second term we have
	\begin{eqnarray*}
		& & \| 1_{\Omega_{-t,s-t}}^{c} \times ( f - (u_{n,t} - s \sigma(v)) ) \|_{1} \\
		& \leq & \| 1_{\Omega_{-t,s-t}}^{c} \times ( \sum_{i=n+1}^{\infty} a_i \phi_i) \|_{1} + \| 1_{\Omega_{-t,s-t}}^{c} \times ( t - s\sigma(v)) \|_{1} \\
		& \leq & \| 1_{\Omega_{-t,s-t}}^{c} \times ( \sum_{i=n+1}^{\infty} a_i \phi_i) \|_{1} + \max(t,|s-t|) \times \mu ({\Omega_{-t,s-t}}^{c}).
	\end{eqnarray*}
	By combining both inequalities and subsequently taking the infimum over $\mathcal{V}$ we arrive at the following:
	\begin{eqnarray*}
		\min_{v \in \mathcal{V}} \| f - (u_{n,t} - s \sigma(v)) \|_{1} & \leq & \| 1_{\Omega_{-t,s-t}}^{c} \times ( \sum_{i=n+1}^{\infty} a_i \phi_i) \|_{1} + \max(t,|s-t|) \times \mu ({\Omega_{-t,s-t}}^{c}) \\
		& & + ~~ s L \inf_{v \in \mathcal{V}} \left \| 1_{\Omega_{-t,s-t}} \times \left( v - \sigma^{-1} \left( \frac{t - \sum_{i=n+1}^{\infty} a_i \phi_i }{s} \right) \right) \right\|_{1}.
	\end{eqnarray*}
	Next, by applying the Cauchy-Schwarz inequality we have the bound
	\begin{equation*}
	\| \sum_{i=n+1}^{\infty} 1_{\Omega_{-t,s-t}^{c}} \times a_i \phi_i \|_{1} \leq \| 1_{\Omega_{-t,s-t}^{c}} \|_{2} \| \sum_{i=n+1}^{\infty}  a_i \phi_i \|_{2} = \sqrt{\mu( \Omega_{-t,s-t}^{c})} \times \| \sum_{i=n+1}^{\infty}  a_i \phi_i \|_{2}.
	\end{equation*}
	By applying Chebyshev's inequality appropriately we arrive at the bounds
	\begin{equation*}
	\mu( \Omega_{-t,s-t}^{c}) \leq \left(\frac{1}{t^2} + \frac{1}{(s-t)^2} \right) \| \sum_{i=n+1}^{\infty}  a_i \phi_i \|_{2}^{2}.
	\end{equation*}
	The result follows by combining these bounds.
\end{proof}

\noindent
\textbf{Remark}:
Our result shows that, under the assumption that the series \eqref{eq:f_decompseries} converges in the L2-norm, the False Negative rate is effectively governed by how quickly the series $ \sum_{i=n+1}^{\infty} a_i \phi_i $ vanishes as well as the size of the offset $t$.  Our bound of the error $\|f-\hat{f}\|$ in Proposition \ref{thm:l1_violatesatefy} comprises two key components.  The first quantity depends on the expressive power of $\mathcal{V}$ in approximating the function $f$, and it may potentially increase with larger $s$.  The second quantity captures the residue of the function $f$ that cannot be captured by $s\sigma(v)$.  In contrast with the first quantity, it is independent of $\mathcal{V}$ and it depends \emph{inversely} with the parameter $s$.


\subsection{Summary and Practical Implications} 
\label{sec:practical_imp}
To summarize, Prop~\ref{thm:approxpower} indicates for arbitrary on-device structure $u \in \mathcal{U}$, by collaborating with on-server network $v \in \mathcal{V}$ and selecting a proper parameter $s$, the model decomposition scheme is always guaranteed to perform no worse than the original complex model $v^*$. With further Assump~\ref{assum} that is applicable to most deep learning models, the on-device model structure can be precisely truncated from the original model structure $\mathcal{V}$. We illustrate applications of the bounds in Propositions \ref{thm:residualbound_inf} and \ref{thm:fp_l1bound} with a series of examples.  

\vspace{0.1in}
\noindent
\textbf{General Case}: In the general case, the on-device network can be truncated as $u_{n,t(n)}$ from the original model, with an augmented term $t$ to capture the residuals as shown in (\ref{eq:un_truncateplusconst}). By combining the results of Prop~\ref{thm:residualbound_inf} and Prop~\ref{thm:fp_l1bound}, it's sufficient to set $t=t(n):=\| \sum_{i=n+1}^{\infty} a_i \phi_i \|_{\infty}$ and $s = 2t$ to obtain ideal approximation with the lowest FP rate, and the FN rate is strictly 0 in this scenario.

\noindent
\textbf{Exponential Decay}: In case the coefficients in (\ref{eq:f_decompseries}) decays in a specific trend, for example the coefficients $a_i$ have exponential decay, i.e.  $a_i = \rho^{i-1}$ for some fixed $\rho\in(0,1)$ with $\sup_{\phi\in\Phi} \| \phi \|_{\infty} < \infty$, we have more concrete results. In this case, Prop~\ref{thm:residualbound_inf} says that picking $s \sim \rho^{n} / (1-\rho)$ ensures both the positivity criterion and accurate approximation. Moreover, Prop~\ref{thm:fp_l1bound} suggests that $\mu_{\mathrm{FP},\epsilon} \lesssim (s + 2\rho^{n}/(1-\rho))/\epsilon$.

\noindent
\textbf{Power-law Decay}: Suppose instead that $a_i$ have power-law decay, i.e.  $a_i = (1/i)^{\alpha}$ for some fixed $\alpha > 1/2$. Assume further that the $\phi_i$'s in~\eqref{eq:f_decompseries} are orthonormal. First we have the bound $ \| \sum_{i=n+1}^{\infty}  a_i \phi_i \|_{2}^{2} \lesssim 1/n^{2\alpha-1}$. Thus, the upper bound on the approximation error is $\mathcal{O}(1/n^{2\alpha-1})$, and we may set $s \approx 1/n^{2\alpha-1}$. The false positive rate has estimate $\mu_{\mathrm{FP},\epsilon} \lesssim (s + 2 /n^{2\alpha-1})/\epsilon$.

\section{Experimental Validations}
In this part, we conduct a series of experiments to validate our theoretical analysis in the previous section.

\begin{figure*}[!t]
	\centering
	\subfigure[Approximation error for $\hat{f}$ on server. ]
	{\label{syn_loss}
		\includegraphics[width=.42\linewidth]{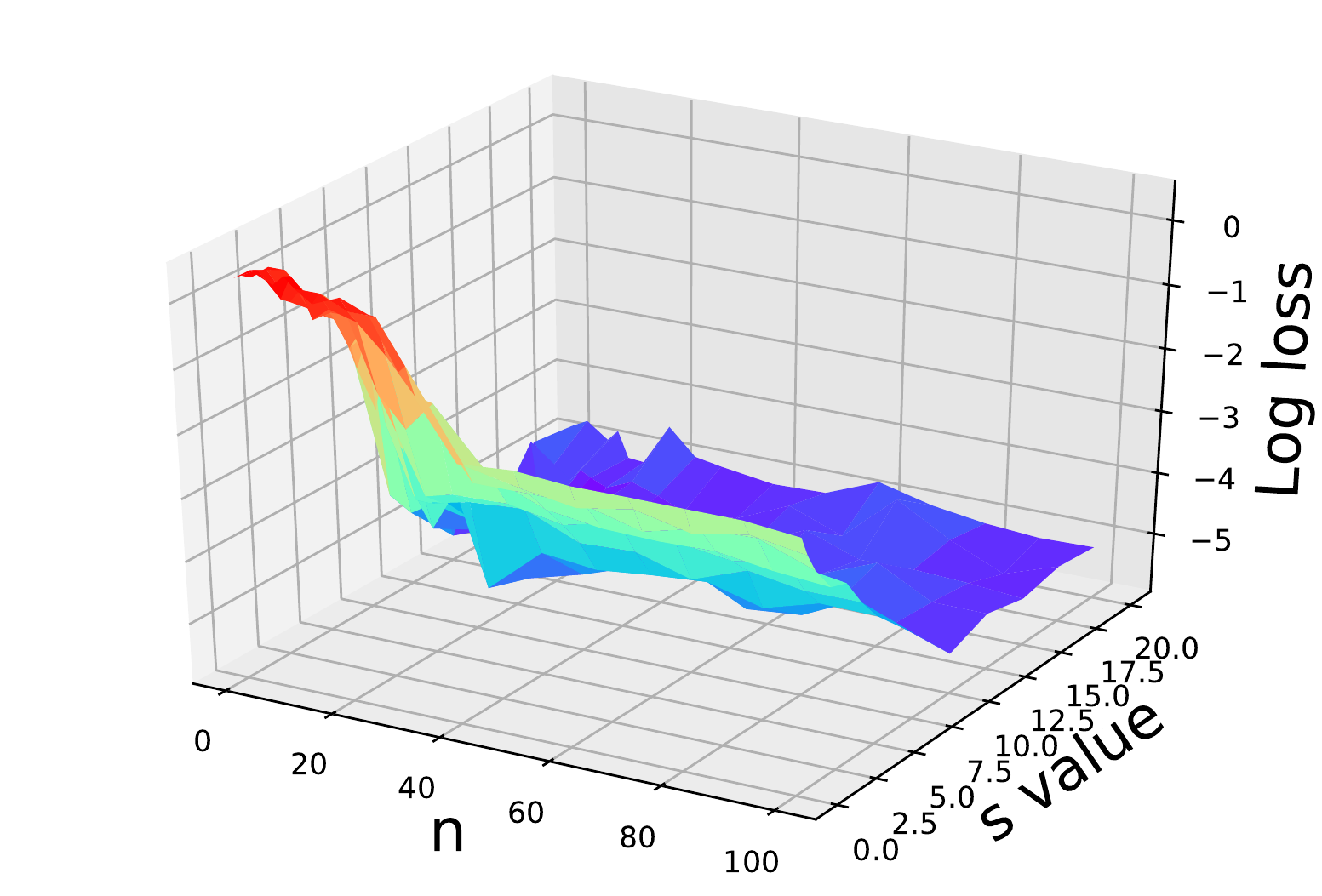}
	}
	\subfigure[FN rate for $u$ on device.]
	{\label{syn_u_FN}
		\includegraphics[width=.42\linewidth]{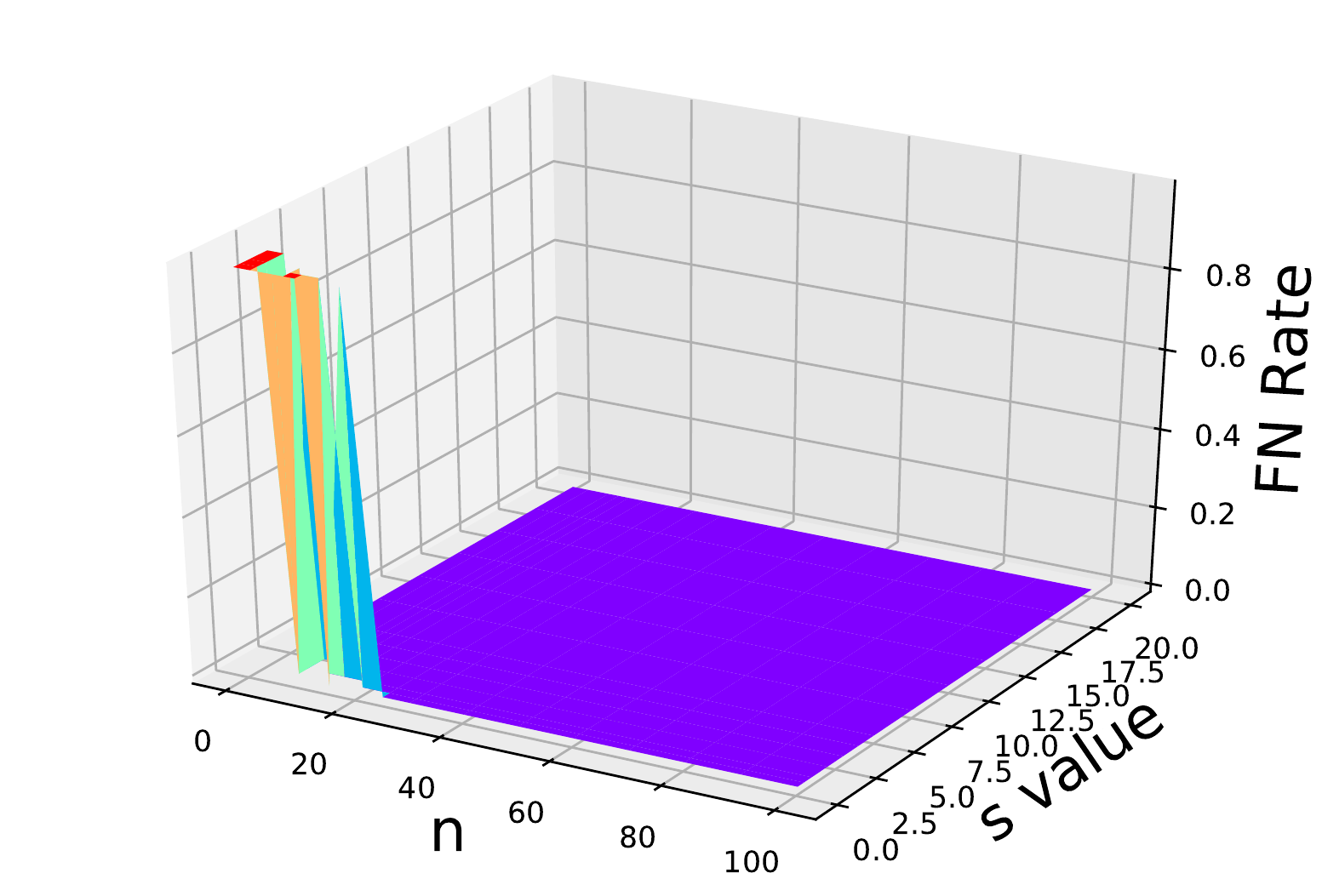}
	}
	\subfigure[FP rate for $u$ on device.]
	{\label{syn_u_FP}
		\includegraphics[width=.42\linewidth]{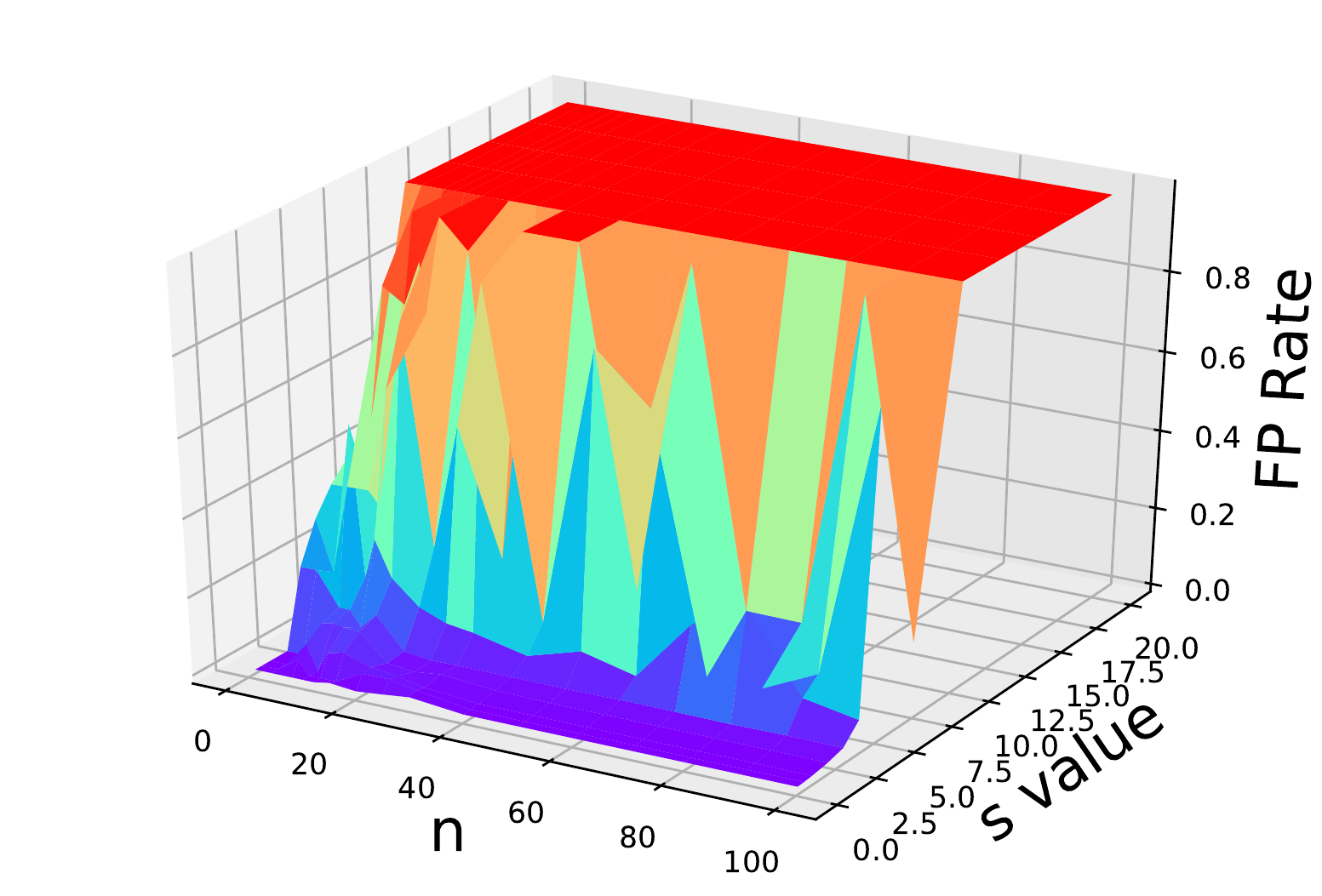}
	}
	\subfigure[FP rate for $\hat{f}$ on server.]
	{\label{syn_FP}
		\includegraphics[width=.42\linewidth]{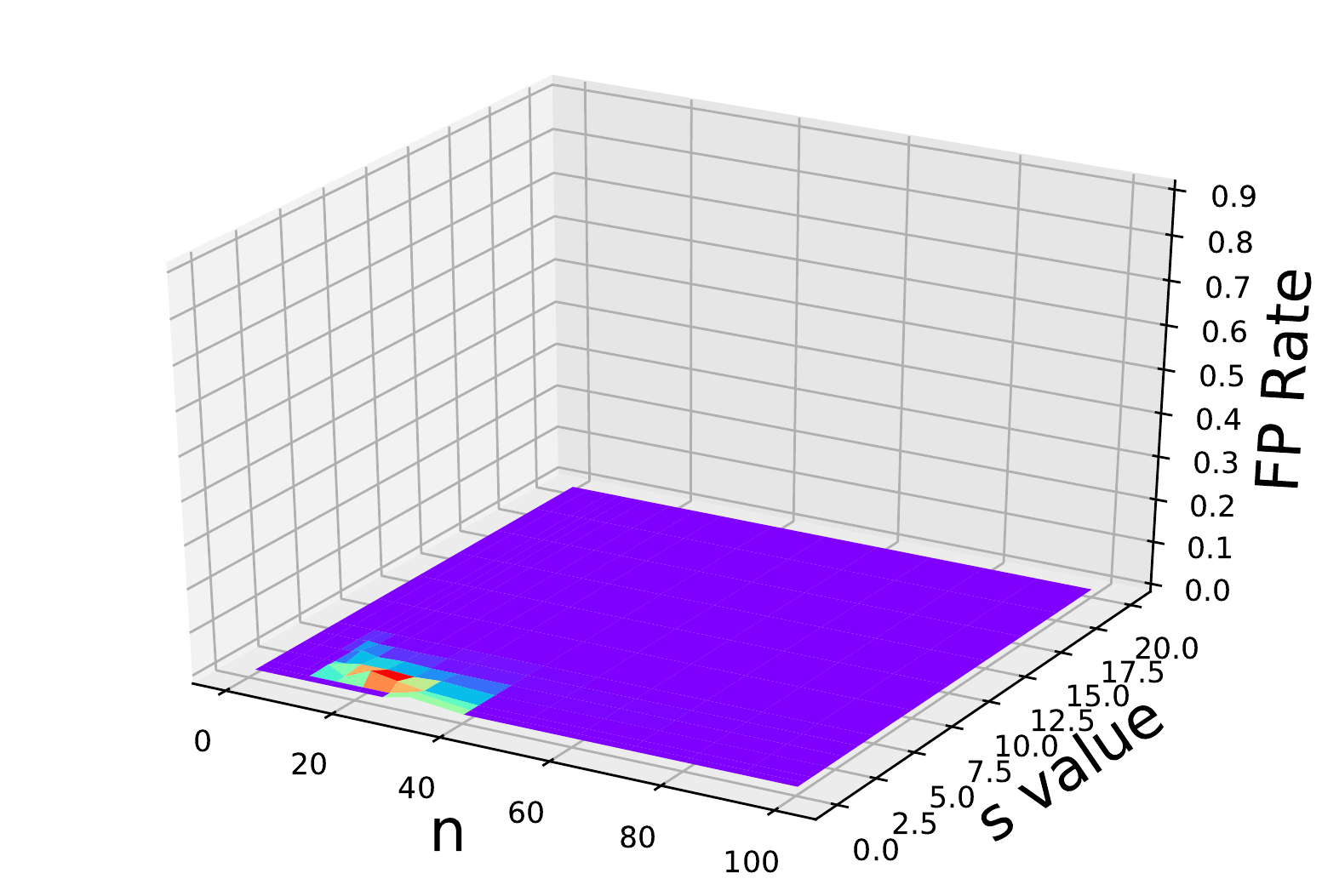}
	}
	\caption{Synthetic experiments.}
	\label{fig:Syn}
\end{figure*}
\subsection{Synthetic Dataset}
The first experiment we consider is a synthetic dataset to simulate the exponential decay case in Sec~\ref{sec:practical_imp}. The goal of this simulation is to validate our analysis in previous section, as well as examining the model efficiency. 

The dataset is generated by sampling input $x \in \mathbb{R}$ from a uniform distribution $\mathbb{U}[-3,3]$ and computing its label as
\begin{equation*}
f(x) = \sum\nolimits_{i=1}^{100}  0.9^{i-1} \cos(i  x),
\end{equation*}
which is equivalent to setting $\phi_i$'s to cosine functions at different frequencies in Eq~(\ref{eq:f_decompseries}) and $a_i=0.9^{i-1}$. In practice, a FC(1,16,32,64,100,1) neural network structure $\mathcal{V}$ can efficiently train the loss to approximately 0, hence we select it as our on-server structure. The on-device monitoring structure $\mathcal{U}$ is selected by truncating and augmenting $\mathcal{V}$ as Eq~(\ref{eq:un_truncateplusconst}). Once the on-device and on-server architectures are selected, the entire model $\hat{f} = u - s \sigma(v)$ is trained end-to-end using the Adam optimizer~\cite{kingma2014adam}.

We first validate the analysis presented in Sec~\ref{sec:analysis_ae} to \ref{sec:analysis_fn} by showing the landscape of the three metrics in Fig~\ref{fig:Syn}. 
\begin{enumerate}
	\item Prop~\ref{thm:residualbound_inf} says we can  obtain good approximation in Eq~(\ref{eq:uniformbound_allcases}) when the condition is satisfied: $s \geq 2t=2\| \sum_{i=n+1}^{\infty} a_i \phi_i \|_{\infty}$. This could be achieved either by choosing a large $s$ or increasing $n$ (hence decreasing $t$), which explains why in Fig~\ref{syn_loss} the loss significantly drops down either we use a large $n$  or a large $s$.
	
	\item For safety concerns, the FN rate should be minimal, and Fig~\ref{syn_u_FN} indicates the FN rate is almost 0 everywhere except for both small $n$ and $t$ where the requirement $s \geq 2t$ is violated.
	
	\item Prop~\ref{thm:fp_l1bound} predicts the on-device FP rate climbs up as we increase $s$, and a similar phenomenon can be observed in Fig~\ref{syn_u_FP}. But for all choices of $s$, these FP cases can be effectively corrected on server, and therefore is FP rate is almost 0 everywhere in Fig~\ref{syn_FP}.
\end{enumerate}

\begin{figure}[!h]
	\centering
	\includegraphics[width=.6\linewidth]{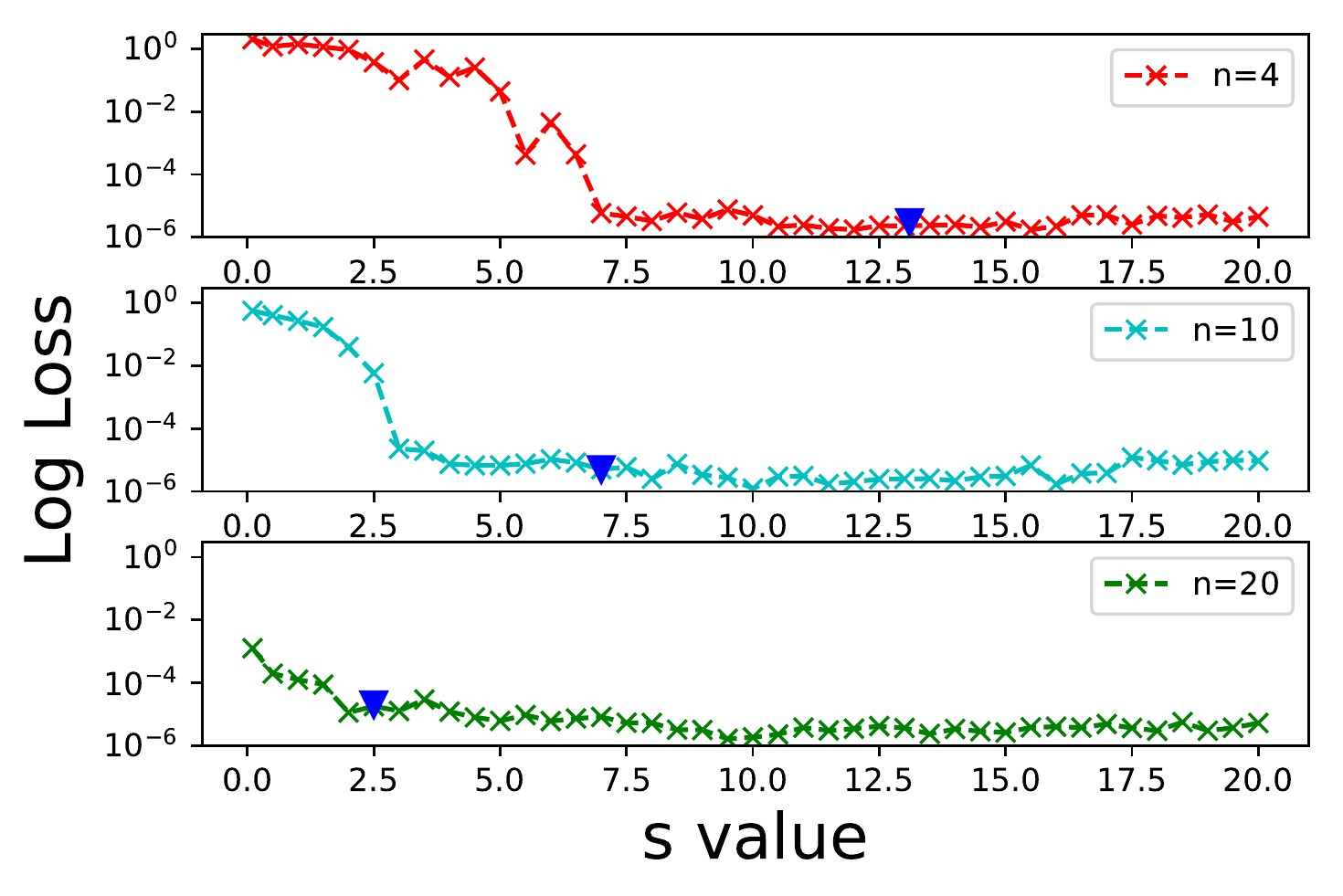}
	\vspace{-0.2in}
	\caption{Approximation Error for $\hat{f}$ w.r.t choices of $s$. The blue triangle represents the suggested theoretical value.}
	\label{fig:Syn_loss}
\end{figure}

The above results show that the decomposed model can achieve a low approximation error with safety guarantee (0 false negative rate) with proper choices of $(s,n)$. Indeed,  our analysis in Sec~\ref{sec:practical_imp} provides a more concrete choice of $s \sim \rho^{n} / (1-\rho)$. Fig~\ref{fig:Syn_loss} shows how the practical approximation error goes with different $s$, and the results indicate that the theoretical choices all obtain good approximation for these three cases. Since we use $\sum_{i=n+1}^{100} |a_i| $ to approximate $\| \sum_{i=n+1}^{100} a_i \phi_i\|_{\infty}$, there is a gap between the theoretical and practical optimal value.

\subsection{Financial Dataset}
The above experiment validates the possibility of our model decomposition strategy, and shows our analysis developed in Sec~\ref{sec:analysis} can guide our practical architecture designs.  In the following parts, we further demonstrate the broad applicability of our proposed decomposition scheme on a real-world dataset. 

\begin{figure*}[!t]
	\centering
	\includegraphics[width=0.88\linewidth]{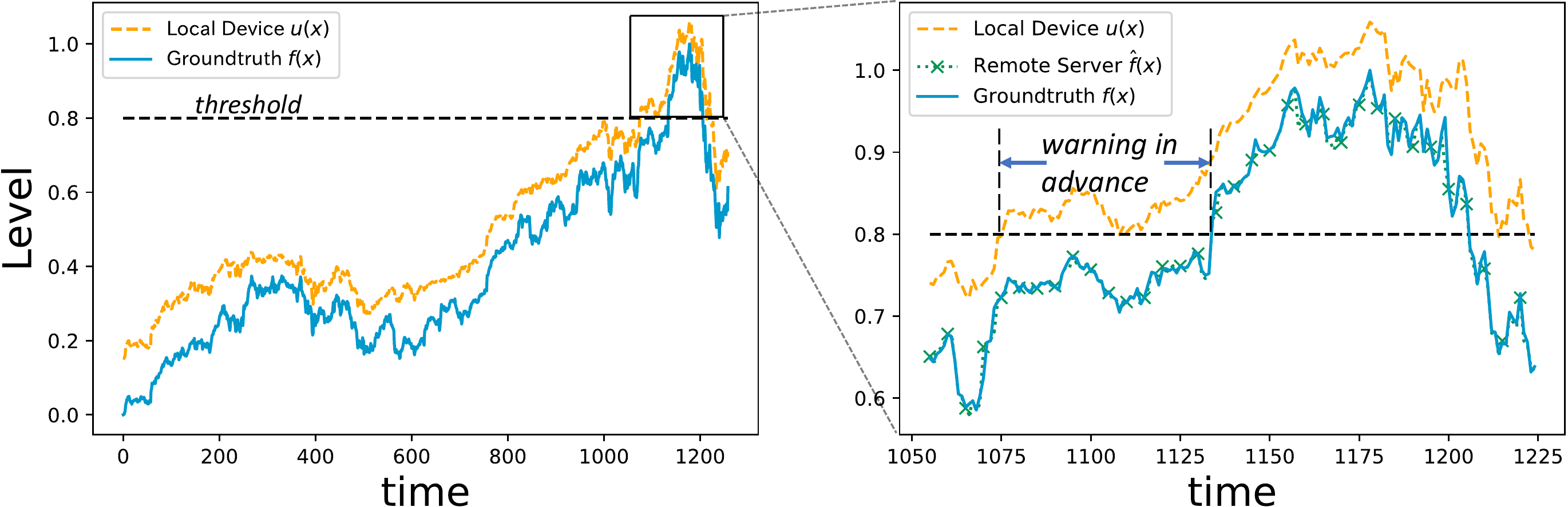}
	\vspace{-0.1in}
	\caption{Financial experiments. The on-device neural network $u$ provides continuous monitoring while the on-server negative corrector $-s \sigma(v)$ is activated only when $u$ is above the threshold. On-device model size is compressed 6x times and the overall communication cost is reduced 10x times.}
	\label{fig:finance}
\end{figure*}

The Dow Jones Industrial Average (DJIA) dataset~\cite{finance} represents the stock market index of 30 companies' stock value and how they have traded in the stock market during various periods of time. We consider Apple Inc.'s stock price as the potential ground truth $f(x)$ and use the price from the other 29 companies to predict it. All data are normalized to range $[0,1]$, with the level $0.8$ as the warning threshold. The baseline architecture ($\mathcal{V}$) is selected as a FC(29,64,128,256,1) neural network, which attains a mean-squared loss of approximately $10^{-6}$. We truncate the neural numbers on the second-last layer to 16 to obtain the on-device model $u(x)$.

\vspace{0.1in}
Fig~\ref{fig:finance} reports the results in the following aspects.
\begin{enumerate}
	\item The on-device predictor $u(x)$ can always provide an upper approximation of the true signal $f$, leading to an early warning before the ground truth exceeds the threshold. This naturally guarantees the safety requirement, since the false negative rate is always 0 in this case.
	\item  In case the on-device provides signal exceeds the threshold, the negative corrector $-s \sigma(v(x))$ is activated and the server provides more accurate predictions $\hat{f}(x) = u(x) -s \sigma(v(x))$. In Fig~\ref{fig:finance}, this combined approximation is almost identical to the groundtruth, hence the whole system obtains more accurate approximation and the extra false positive prediction is eliminated by remote corrector.
	\item By local monitoring and local analysis $u(x)$, the communication cost is reduced 10x times.
\end{enumerate}

In addition to strictly truncate model based on Prop~\ref{thm:residualbound_inf}, Prop~\ref{thm:approxpower} also indicates a simple neural network can also act as on-device upper approximator. In appendix, we use a FC(29,10,1) network as a local monitoring tool so that the on-device model size is compressed more significantly. The drawback is we have to manually select a larger $s$ and incur slightly more false positive cases (+2\%) for $u$.

\begin{figure*}[!h]
	\centering
	\includegraphics[width=1\linewidth]{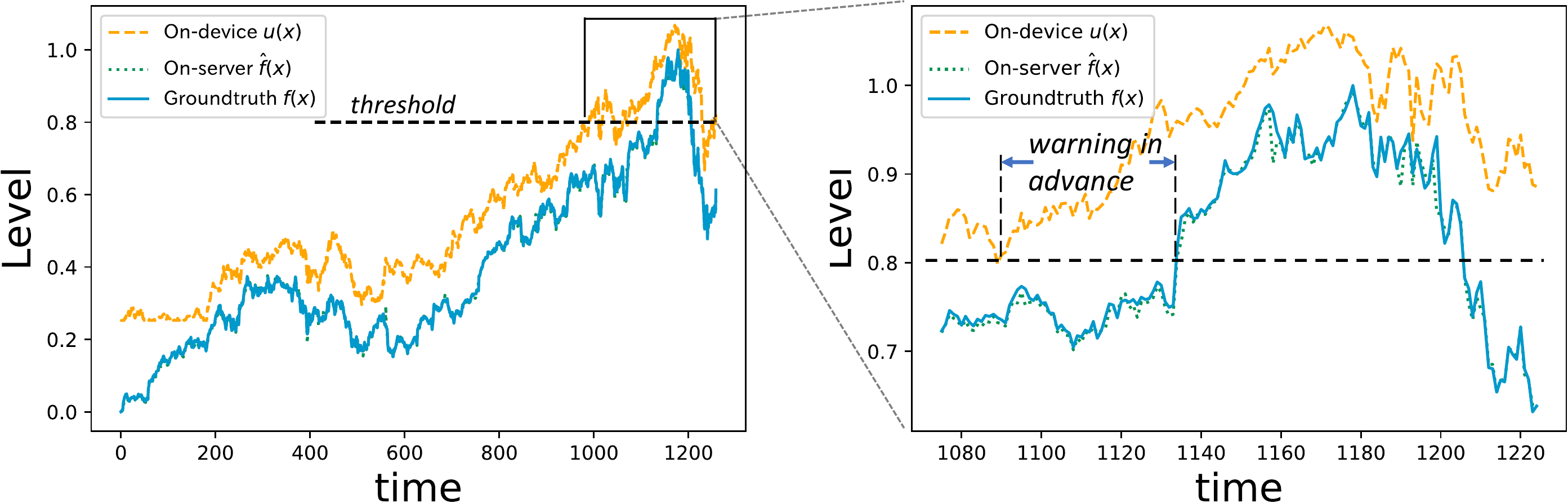}
	\caption{Finance dataset with a simpler monitoring net.}
	\label{fig:financesimple}
\end{figure*}

To summarize, the model decomposition scheme allows us to use collaborative inference instead of purely on-device or on-server predictions. The approximation error is shown to be no worse than the previous complex model in both theory and practice. Specifically, the safety regime proposed in this paper makes our work different from previous studies. 
For financial dataset, this upper approximation allows the users to buy in the stock price in advance. For other scenarios like health monitoring, this would allow us to predict the disease in advance, which could be more vital than purely proving an accurate approximation.

\vspace{0.1in}
\section{Conclusion}

In this paper, we introduce a model decomposition scheme to realize efficient and safe inference for remote monitoring tasks on edge devices. The key idea is the combination of a simple local monitoring surrogate and a complex negative corrector on the server side.  We demonstrate using experiments that following our theoretical analysis, one can greatly decrease the model complexity required for safe monitoring, thereby increasing the applicability of deep learning models in resource-constrained and safety-sensitive applications such as remote monitoring.


\clearpage
\bibliography{CPI.bib}
\bibliographystyle{plain}

\end{document}